\documentclass[conference]{IEEEtran}
\usepackage[cmex10]{amsmath}
\usepackage{amssymb}
\usepackage{array}
\usepackage{mdwmath}
\usepackage{mdwtab}
\usepackage{enumerate}
\usepackage{eqparbox}
\usepackage{graphicx}
\usepackage[tight,footnotesize]{subfigure}
\usepackage{fixltx2e}
\usepackage{url}
\usepackage{algorithm} 
\usepackage{slashbox}
\usepackage{algpseudocode}
\usepackage[section]{placeins}
\usepackage{helvet}
\newtheorem{theorem}{Theorem}
\newtheorem{lemma}{Lemma}
\usepackage{color}

\begin{document}
%
\title{Concept Drift Detection for Streaming Data}

\author{\IEEEauthorblockN{Heng Wang}
\IEEEauthorblockA{
Johns Hopkins University\\
Email: hwang82@jhu.edu}
\and
\IEEEauthorblockN{Zubin Abraham}
\IEEEauthorblockA{Robert Bosch LLC\\ Research and Technology Center North America\\
Email: Zubin.Abraham@us.bosch.com}
}

\maketitle

\begin{abstract}
Common statistical prediction models often require and assume stationarity in the data. However, in many practical applications, changes in the relationship of the response and predictor variables are regularly observed over time, resulting in the deterioration of the predictive performance of these models. This paper presents Linear Four Rates (LFR), a framework for detecting these concept drifts and subsequently identifying the data points that belong to the new concept (for relearning the model). Unlike conventional concept drift detection approaches, LFR can be applied to both batch and stream data; is not limited by the distribution properties of the response variable (e.g., datasets with imbalanced labels); is independent of the underlying statistical-model; and uses user-specified parameters that are intuitively comprehensible. The performance of LFR is compared to benchmark approaches using both simulated and commonly used public datasets that span the gamut of concept drift types. The results show LFR significantly outperforms benchmark approaches in terms of recall, accuracy and delay in detection of concept drifts across datasets.
\end{abstract}


%

\section{Introduction}
A common challenge when mining data streams is that the data streams are not always strictly stationary, i.e., the concept of data (underlying distribution of incoming data) unpredictably drifts over time. This has encouraged the need to detect these concept drifts in the data streams in a timely manner, be it for business intelligence or as a means to track the performance of statistical prediction models that use these data streams as input.

This paper focuses on detecting concept drifts affecting binary classification models. For a binary classification problem, concept drift is said to occur when the joint distribution $P(\mathbf{X}_t, y_t)$ changes over time where $\mathbf{X}_t\in \mathbb{R}^d$ are the $d$ predictor variables at time step $t$ and $y_t\in\{0,1\}$ the corresponding binary response variable.  Intuitively, concept drift refers to the  scenario when the underlying distribution that generates the response variable changes over time. Popular approaches for detecting concept drift identify the change point \cite{gama2004learning,wang2013concept}. DDM is the most widely used concept drift detection algorithm, that is strictly designed for streaming data \cite{gama2004learning}. The test statistic DDM employs is the sum of overall classification error ($\hat{P}^{(t)}_{error}$) and its empirical standard deviation ($\hat{S}^{(t)}_{error}$). DDM focuses on the overall error rate and hence fails to detect a drift unless the sum of false positive and false negatives changes. An example of such a scenario, is when a $2\times2$ confusion matrix changes from 
\begin{small}
$ \left( \begin{array}{cc}
			65 & 5 \\
			15 & 15 \end{array} \right)$ to 
			$\left( \begin{array}{cc}
			75 & 15 \\
			5 & 5 \end{array} \right),$
\end{small}
 thus preserving their overall error rate.
This limitation is accentuated in imbalanced classification tasks \cite{wang2013concept}, as seen in the example. Unfortunately, this failure to detect a drastic drop in recall of the minority class is often critical. For instance, if the minority class in the above example corresponded to products at a manufacturing plant that were classified as defective, this critical threefold decrease in 'true positive rate' (i.e., from 0.75 to 0.25) would go unnoticed by DDM.

Drift Detection Method for Online Class Imbalance (DDM-OCI) addresses the limitation of DDM when class ratio is imbalanced \cite{wang2013concept}. However, DDM-OCI triggers a number of false alarms due to an inherent weakness in the model. DDM-OCI assumes that the concept drift in an imbalanced classification task is indicated by the change of underlying true positive rate (i.e., minority-class recall). This hypothesis unfortunately does not consider the case when concept drift occurs without affecting the recall of the minority class. It can be shown that it is possible for concept to drift from an imbalanced class data to balanced class data, while true positive rate ($tpr$), positive predicted value ($ppv$) and F1-score remain unchanged. Thus, this type of drift is unlikely to be detected by DDM-OCI unless other rates such as true negative rate ($tnr$) or negative predicted value ($npv$) are also considered. Additionally, the test statistic used by DDM-OCI $R^{(t)}_{tpr}$ is not approximately distributed as $\mathcal{N}(P_{tpr}, \dfrac{P_{tpr}(1-P_{tpr})}{N_{tpr}})$, under the stable concept. Thus, the rationale of constructing confidence levels specified in \cite{gama2004learning} is not suitable with the null distribution of $R^{(t)}_{tpr}$. This is the reason DDM-OCI triggers false alarms quickly and frequently. 

Early Drift Detection Method (EDDM) achieves better detection results than DDM if the data stream has slow gradual change. EDDM monitors the distance between the two classification errors \cite{baena2006early}. PerfSim algorithm considers all the components of a confusion matrix and monitors the cosine similarity coefficient of all components from two batches of data \cite{antwi2012perfsim}. If the similarity coefficient drops below some user-specified threshold, a concept drift is signified. However, EDDM requires to wait for a minimum of $30$ classification errors before calculating the monitoring statistic at each decision point. That is, the length of a time interval between decision points of a drift is a random number depending on $30$ appearances of classification errors. It is possible that there is a great many examples between $30$ classification errors. PerfSim algorithms is also constrained by the requirement for collecting mini-batch data to calculate monitoring statistics.
The method to partition data stream in \cite{baena2006early, antwi2012perfsim} is either user-specified by practical experience or to be learned before the start of detection. Hence, EDDM and PerfSim are not well suited for streaming environments in which decisions are made instantly. The approach specified in \cite{klinkenberg2000detecting} makes use of SVM to monitor three measures: overall accuracy, recall, and precision over time. This aproach too computes the three measures by assuming that the data arrives in batches, on which SVM is learned.

To address the limitations of existing approaches, we present Linear Four Rates (LFR) for detecting the drift of $P(\mathbf{X}_t, y_t)$. Unlike other proposed approaches, LFR can detect all possible variants of concept drift, even in the presence of imbalanced class labels, as shown in Section \ref{sec:Experiments}. LFR outperforms existing approaches in terms of earliest detection of concept drift, with the least false alarms and best recall. Additionally, LFR does not require the data to arrive in batches and is independent of the underlying classifier employed. 

\section{Problem formulation}
\label{sec:ProblemFormulation}

Given that detection of concept drift is equivalent to detecting a change-point in $P(\mathbf{X}_t, y_t)$, an intuitive approach is to test the statistical hypothesis upon the multivariate variable $(\mathbf{X}_t, y_t)$ in the data stream \cite{matteson2014nonparametric,song2007statistical,dries2009adaptive}. The limitation of this approach is that the performance of the statistical power degrades when the dimension ($d$) of $\mathbf{X}_t$ is extremely large or if the magnitude of the drift small. Hence, to overcome these limitations, the proposed approach identifies the change in $P(\hat{f}(\mathbf{X}_t), y_t)$ where $\hat{f}$ is the classifier used for prediction. This is motivated by the fact that any drift of $P(\hat{f}(\mathbf{X}_t), y_t)$ would imply a drift in $P(\mathbf{X}_t, y_t)$, with probability 1.

Let $\hat{f}(\mathbf{X}_t)= \hat{y}_t$ be a binary classifier for the given data stream ($\mathbf{X}_t, y_t$). We define the corresponding $2\times2$ confusion probability matrix ($CP$) for $\hat{f}$ to be
\begin{table}[h]
\centering
$CP$= 
\begin{tabular}{|c|c|c|}
\hline
\backslashbox{Pred}{True} & 0 &  1 \\
\hline
0& TN & FN \\
\hline
1& FP & TP \\
\hline
\end{tabular}
\end{table}

where, $CP[1,1]$, $CP[0,0]$, $CP[1,0]$, $CP[0,1]$ denotes the underlying percentage of true positives (TP), true negatives(TN), false positives (FP) and false negatives (FN) respectively, for classifier $\hat{f}$. i.e., $CP[1,1]=P(y_t=1, \hat{y}_t=1)$. 

The four characteristic rates (True Positive Rate, True Negative Rate, Positive Predicted Value, Negative Predicted Value) can be computed as follows:  $P_{tpr}=TP/(TP+FN)$,  $P_{tnr}=TN/(TN+FP)$,  $P_{ppv}=TP/(FP+TP)$ and  $P_{npv}=TN/(TN+FN)$. All the mentioned characteristic rates in $P_\star=\{P_{tpr}, P_{tnr}, P_{ppv}, P_{npv}\}$ are equal to $1$, if there is no misclassification. 

Under a stable concept (i.e., $P(\mathbf{X}_t, y_t)$ remains unchanged), $\{P_{tpr}, P_{tnr}, P_{ppv}, P_{npv}\}$ remains the same. Thus, a significant change of any $P_\star$, implies a change in underlying joint distribution $(y_t, \hat{y}_t)$, or concept. It is worth noting that at every time step $t$, for any possible $(y_t, \hat{y}_t)$ pair, only two of the four empirical rates in $P_\star$ will change and these two rates are referred to as ``influenced by $(y_t, \hat{y}_t)$". Also, note that in certain applications the detection of concept drift is not of interest and thus unnecessarily alarmed if all empirical rates in $P_\star$ are increasing. This is because it suggests that an old model learned from historical data performs even better in classifications of current data stream. We do not use this assumption in this paper,  but all methodologies and arguments we propose below can be easily adapted for this assumption.


\section{Concept Drift Detection Framework}
\label{sec:Algorithm}
Given the 
efficacy of the $P_\star$ (where, $\star\in{\{tpr,tnr,ppv,npv\}}$) to detect concept drift, the proposed concept drift detection framework uses estimators of the rates in $P_\star$ as test statistics to conduct statistical hypothesis testing at each time step. 
Specifially, the framework at each time step $t$ conducts statistical tests with the following null and alternative hypotheses: 
\[H_0: \forall \star, P(\text{estimator of }P_{\star}^{(t-1)})=P(\text{estimator of }P_{\star}^{(t)})\]
\[H_A: \exists \star, P(\text{estimator of }P_{\star}^{(t-1)})\neq P(\text{estimator of }P_{\star}^{(t)}).\]
The concept is stable under $H_0$ and is considered to have drifted if $H_0$ is rejected. The idea is to compare the statistical significance level of the running test statistic under $H_0$ at each time step to the user defined warning ($\delta_\star$) and detection ($\epsilon_\star$) significance levels. 
This type of test is called "continuing test" \cite{aroian1950effectiveness} and in our problem all time stamps are decision points of acceptance or rejection. Then when the concept is stable, false alarms on $P_\star$ will be triggered unnecessarily once in every $1/\epsilon_\star$ time steps in the long run. In this paper, we assume the spacing of decision points is fixed. Accordingly, the familiywise error rate and its cost in our continuing test can be controlled by using a simultaneous inference method such as classical Bonferroni corrections on $\epsilon_\star$. In a more general case where the spacings of decision points are unequal and test statistics are strongly positive correlated, we should instead consider the average run length of the test \cite{basseville1993detection} or more powerful alternatives that controls the familywise error rate.

A na\"{i}ve implementation of the "continuing test" framework (Na\"{i}ve Four Rates) would be to use $\hat{P}^{(t)}_\star$ (empirical rate of $P^{(t)}_\star$), as the estimators and test statistics. But as shown in Section \ref{sec:NFRvsLFR}, there are better estimates of $\hat{P}^{(t)}_\star$. 


In the following section, Linear Four Rates (LFR) algorithm will be used to elaborate on the concept drift detection framework. LFR differs from Na\"{i}ve Four Rates (NFR) in terms of the estimator used. However, both LFR as well as NFR perform better than DDM and DDM-OCI due to the more comprehensive detection framework utilized.

\subsection{Linear Four Rates algorithm (LFR)}
\label{sec:LFRalgorithm}
\subsubsection{Algorithm Outline}
\label{LFRoutline}
LFR uses modified rates $R^{(t)}_\star$ as the test statistics for $P_\star^{(t)}$. $R^{(t)}_\star$ is a modified version of the empirical rate $\hat{P}^{(t)}_\star$.
At each $t$,  $R^{(t)}_{\star}$ is updated as : $R^{(t)}_{\star} \leftarrow \eta_{\star}R^{(t-1)}_{\star} + (1-\eta_{\star}) \mathbf{1}_{\{y_t = \hat{y}_t\}}$ for those empirical rates $\star$ ``influenced by $(y_t,\hat{y}_t)$". $R^{(t)}_{\star}$ is essentially a linear combination of classifier's previous performance $R^{(t-1)}_{\star}$ and current performance $\mathbf{1}_{\{y_t = \hat{y}_t\}}$, where $\eta_\star$ is a time decay factor for weighting the classifier's performance at current instance. $R^{(t)}_\star$ has been used as a class imbalance detector and as a revised recall test statistic in \cite{wang2013learning}\cite{wang2013concept}. The probabilistic characteristic of our test statistic $R_\star$ are investigated in \S~\ref{LFRtheory}. The pseudocode of the framework (using $R^{(t)}_\star$ as an estimator of $P_\star^{(t)}$ for required test statistic), is detailed in Algorithm~\ref{FourRatesAlg}.
\begin{algorithm}
	  \caption{Linear Four Rates method (LFR)}
	\label{FourRatesAlg}
	  \begin{algorithmic}[1]
	   \Require 
		Data: $\{(\mathbf{X}_t, y_t)\}_{t=1}^{\infty}$ where $\mathbf{X}_t \in \mathbb{R}^d$ and $y_t \in \{0,1\}$
		Binary classifier  $\hat{f}(\cdot)$;
		Time decaying factors $\eta_{*}$; 
		Warn significance level $\delta_{*}$; 
		Detect significance level $\epsilon_{*}$. 
	   \Ensure
		Detected concept drift time ($t_{cd}$).
	    \State $\hat{P}^{(0)}_{\star} \leftarrow 0.5$, $R^{(0)}_{\star} \leftarrow 0.5$, 
		  where $\star \in \{tpr, tnr, ppv, npv\}$
		and confusion matrix
		$C^{(0)} \leftarrow [1,1;1,1]$;
	    \For  {$t=1$ to $\infty$} 
		\State $\hat{y}_t \leftarrow \hat{f}(\mathbf{X}_t)$
	   \State $C^{(t)}[\hat{y}_t][y_t] \leftarrow C^{(t-1)}[\hat{y}_t][y_t] + 1$
	    \For{\textbf{each} $\star \in \{tpr, tnr, ppv, npv\}$}
		\If{($\star$ is influenced by $(y_t, \hat{y}_t)$)}      
			\State $R^{(t)}_{\star} \leftarrow \eta_{\star}R^{(t-1)}_{\star} + (1-\eta_{\star}) \mathbf{1}_{\{y_t = \hat{y}_t\}}$ 
		\Else
			\State $R^{(t)}_{\star} \leftarrow R^{(t-1)}_{\star}$
		\EndIf

		\If {( $\star \in \{tpr, tnr\}$)}
		\State $N_\star \leftarrow C^{(t)}[ 0, \mathbf{1}_{\{\star = tpr\}}] + C^{(t)}[1, \mathbf{1}_{\{\star = tpr\}}]$
		\State $\hat{P}^{(t)}_\star \leftarrow \dfrac{C^{(t)}[ \mathbf{1}_{\{\star = tpr\}},  \mathbf{1}_{\{\star = tpr\}}]}{N_\star}$
		\Else
		\State $N_\star \leftarrow C^{(t)}[ \mathbf{1}_{\{\star = ppv\}}, 0] + C^{(t)}[ \mathbf{1}_{\{\star = ppv\}}, 1]$
		\State $\hat{P}^{(t)}_\star \leftarrow \dfrac{C^{(t)}[ \mathbf{1}_{\{\star = ppv\}},  \mathbf{1}_{\{\star = ppv\}}]}{N_\star}$
		\EndIf
		\State $\text{warn.bd}_{\star} \leftarrow \text{BoundTable}(\hat{P}^{(t)}_{\star}, \eta_{\star}, \delta_{\star}, N_\star)$ 
		\State $\text{detect.bd}_{\star} \leftarrow \text{BoundTable}(\hat{P}^{(t)}_{\star}, \eta_{\star}, \epsilon_{\star}, N_\star)$ 
	\EndFor
	    \If{(any $R^{(t)}_\star$ exceeds $\text{warn.bd}_\star$ \& warn.time $=$ 0) } 
	    \State $\text{warn.time} \leftarrow t$
	\ElsIf{(no $R^{(t)}_\star$ exceeds $\text{warn.bd}_\star$)} 
	\State $\text{warn.time} \leftarrow$ 0
	    \EndIf  
	    \If{(any $R^{(t)}_\star$ exceeds $\text{detect.bd}_\star$ ) } 
	    \State detect.time $\leftarrow t$;
	    \State relearn $\hat{f}(\cdot)$ using $\{(\mathbf{X}_t, y_t)\}_{t = \text{warn.time}}^{\text{detect.time}}$
	    \State reset $R^{(t)}_{\star}, \hat{P}^{(t)}_{\star}, C^{(t)}$ as done in Step 1
		\State return $t_{cd} \leftarrow t$
	    \EndIf
	    \EndFor
	  \end{algorithmic}
	\end{algorithm}

The three user defined parameters are the time decaying factor ($\eta_\star$), warning significance level ($\delta_\star$) and detection significance level ($\epsilon_\star$) for each rate. Time decaying factor is a weight in $[0,1]$ to evaluate performance of classifier $\hat{f}$ at current instance prediction $\hat{f}(\mathbf{X}_t)$. Given that the detection methodology is conducting hypothesis testing at each time step, $\delta_\star$ and $\epsilon_\star$ are interpretable statistical significance levels, i.e., type I error (false alarm rate), in standard testing framework. In practice, allowable false warning rate and false detection rate in applications such as quality control of the moving assembly line are guidelines to help the user choose the parameters $\delta_\star$ and $\epsilon_\star$. For the fair comparison, $\eta_\star$ is set to the same value of 0.9 as in \cite{wang2013concept}, for all experiments of this paper. The optimal selection of $\eta_\star$ is domain dependent and can be pre-learned if necessary.


Theorem \ref{LFRnulldist} in Section \ref{LFRtheory} shows that under the stable concept, $R^{(t)}_{\star}$ is a geometrically weighted sum of i.i.d Bernoulli random variables, which emphasizes the most recent prediction accuracy and places exponentially decaying weights on the historical prediction accuracies. By taking advantage of this weighting scheme, $R^{(t)}_{\star}$ is more sensitive to concept drifts, foreshadowing the non-stationarity of classifier's performance.

Standing on Theorem \ref{LFRnulldist}, we are able to overcome the shortcoming of \cite{wang2013concept} and construct a more reliable running confidence interval for $R_\star^{(t)}$ to control the type-I error $\epsilon_\star$. $R^{(t)}_\star$ is distributed as geometrically weighted sum of Bernoulli random variables. Bhati et. al investigates the closed-form distribution function of $R^{(t)}_\star$ for the special case $P_\star=0.5$ \cite{bhati2011distribution}. However, a closed-form distribution function for other values of $P_\star$ is unattainable. Alternatively, according to Theorem \ref{LFRnulldist}, a reasonable empirical distribution can also be independently obtained by Monte Carlo simulation for given $P_\star$, $N_\star$ and time decaying factor $\eta$. The pseudocode for the Monte Carlo sampling procedure is provided in Algorithm \ref{BoundTable}. As $P_\star$ is unknown, $\hat{P}_\star$ is used as its surrogate to generate the empirical distribution of $R^{(t)}_\star$. Based on the empirical distribution, the lower and upper quantile for the given significance level $\alpha$, serves as the required (warning/detect) bounds. The selection of $\hat{P}_\star$ as the best surrogate of $P_\star$, is supported by Lemma \ref{UMVUElemma}.

$\delta_\star$ and $\epsilon_\star$ denote warning and detection significance levels respectively, where $\delta_\star>\epsilon_\star$. The corresponding $warn.bd$ and $detect.bd$ are obtained from Monte Carlo simulations as described. The bounds of four rates $\{tpr, tnr, ppv, npv\}$ of the framework, can be independently set based on importance, by having distinct $\epsilon_\star$. For instance, in some imbalanced classification tasks, performance of the classifier on the minority class is a higher priority than on the majority class. 

Having computed the bounds, the framework considers that a concept drift is likely to occur and sets the warning signal ($warn.time \leftarrow t$), when any $R^{(t)}_\star$ crosses the corresponding warning bounds ($warn.bd$) for the first time. If any $R^{(t)}_\star$ reaches the correspoinding detection bound ($detect.bd$), the concept drift is affirmed at ($detect.time \leftarrow t$).

All examples stored between $warn.time$ and $detect.time$ are extracted to relearn a new classifier since the stored examples are considered samples of the new concept. In case the number of stored examples is too few to relearn a reasonable classifier, one will have to wait for sufficient training examples. However, if $R^{(t)}_\star$ cross the corresponding warning bounds $warn.bd$, but fail to reach $detect. bd$, previous warning flag will be erased. After detecting concept drift, $R^{(t)}_{\star}, \hat{P}^{(t)}_{\star}, C^{(t)}$ are reset to their initial values, so that a new monitoring cycle can restart.

	\begin{algorithm}
	  \caption{Generation of BoundTable in LFR algorithm}
	\label{BoundTable}
	  \begin{algorithmic}[1]
	   \Require 
		Estimate of underlying rate $\hat{P}$;
		Time decaying factor $\eta$; 
		Significance level $\alpha$; 
		Number of time steps $N_\star$;
		Number of random variables $num.of.MC$;
	   \Ensure
		Numeric bound for significance level $\alpha$.
	    \For {$j = 1$ to $num.of.MC$}
			\State Generate $N_\star$ independent Bernoulli random variables $\{I_1,I_2,\dots, I_{N_\star}\}$ where $I_i \stackrel{iid}{\sim} Bernoulli(\hat{P})$
			\State $R[j] \leftarrow (1-\eta)\sum_{i=1}^{N_\star}\eta^{N_\star - i}I_i$
		\EndFor
		\State $\{R[j]\}_{j=1}^{num.of.MC}$ forms a empirical distribution $\hat{F}(R)$ , find $\alpha-$level quantile as the lower bound $lb \leftarrow quantile(\hat{F}(R), \alpha)$ and  $(1-\alpha)-$level quantile as the upper bound $ub \leftarrow quantile(\hat{F}(R), 1-\alpha)$
	  \end{algorithmic}
	\end{algorithm}

\subsubsection{Analysis}
The following theorems investigate the statistical properties of LFR test statistic $R^{(t)}_\star$. 

\label{LFRtheory}
\begin{theorem}
\label{LFRnulldist}
For any $\star$, $R_\star^{(T)}$ is a geometrically weighted sum of Bernoulli random variables, when there is a stable concept up to time $T$: i.e., $R_\star^{(T)} = (1-\eta_\star)\sum_{i=1}^{N_\star}\eta_\star^{N_\star-i} I_i$, where $\{I_i\}_{i=1}^{N_\star} \stackrel{iid}{\sim} Bernoulli(P_\star)$ and $P_\star$ is the underlying rate. 
 \end{theorem} 
\begin{proof}
Among total $T$ time steps, suppose $R^{(t)}_\star$ is changed according to line 7 at time step $T_1,\dots, T_{N_\star}$ where $T_1<T_2<\dots<T_{N_\star}\leq T$. Hence, 
\begin{equation*} \label{eq1}
\begin{split}
R^{(T)}_{\star} = &R^{(T_{N_\star})}_{\star}  
 = \eta_{\star}R^{(T_{N_\star-1})}_{\star} + (1-\eta_{\star}) \mathbf{1}{\{y_{T_{N_\star}} = \hat{y}_{T_{N_\star}}\}} \\
  = &\eta_{\star}[ \eta_{\star}R^{(T_{N_\star-2})}_{\star} + (1-\eta_{\star}) \mathbf{1}{\{y_{T_{N_\star}-1} = \hat{y}_{T_{N_\star}-1}\}} ] \\
	& + (1-\eta_{\star}) \mathbf{1}{\{y_{T_{N_\star}} = \hat{y}_{T_{N_\star}}\}} \\
= &\eta_{\star}^2 R^{(T_{N_\star-2})}_{\star} + \eta_{\star}(1-\eta_{\star}) \mathbf{1}{\{y_{T_{N_\star}-1} = \hat{y}_{T_{N_\star}-1}\}} \\
	& + (1-\eta_{\star}) \mathbf{1}{\{y_{T_{N_\star}} = \hat{y}_{T_{N_\star}}\}} \\
= & \cdots \\
= & (1-\eta_\star)\sum_{i=1}^{N_\star} \eta^{N_\star-i}\mathbf{1}{\{y_{T_i}=\hat{y}_{T_i}\}} \\
= & (1-\eta_\star)\sum_{i=1}^{N_\star}\eta_\star^{N_\star-i} I_i
\end{split}
\end{equation*}
where the last equation hold by the stable concept assumption and all indicators are i.i.d Bernoulli random variables with underlying rate $P_\star$.
\end{proof}

\begin{lemma} 
\label{UMVUElemma}
Assume the setting in Theorem \ref{LFRnulldist}. Under the stable concept, for any  $\star \in \{tpr, tnr, ppv, npv\}$, $\hat{P}^{(T)}_\star$ is the unique Uniformly Minimum Variance Unbiased Estimator (UMVUE) of $P_\star$. As $T\to \infty$, $\hat{P}^{(T)}_\star$ is approximately distributed as $\mathcal{N}(P_\star, \dfrac{P_\star(1-P_\star)}{N_\star})$.
\end{lemma}
\begin{proof}
 $\hat{P}^{(T)}_\star$ is an unbiased estimator of $P_\star$. This is because $\hat{P}^{(t)}_\star = \dfrac{\sum_{i=1}^{N_\star}X_{T_i}}{N_\star}$ where $\{X_{T_i}\}_{i=1}^{N_\star}$ are i.i.d Bernoulli random variables realized at time $T_i$ with parameter $P_\star$. By factorization theorem, $\hat{P}^{(t)}_\star$ is a sufficient statistic. Also, 
\begin{equation*}
\begin{split}
E(g(\hat{P}^{(T)}_\star)) =& \sum_{i=0}^{N_\star} {N_\star \choose i}P_\star^i(1-P_\star)^{N_\star-i}g(\dfrac{i}{N_\star}) \\
=&N_\star!(1-P_\star)^{N_\star}\sum_{i=0}^{N_\star}\dfrac{g(i/N_\star)}{i!(N_\star-i)!}(\dfrac{P_\star}{1-P_\star})^{i} 
\end{split}
\end{equation*}
If $E(g(\hat{P}^{(T)}_\star))=0~\forall P_\star$, it implies $g(\dfrac{i}{N_\star})=0~ \forall i$ because $E(g(\hat{P}^{(T)}_\star))$ is a polynomial of $\dfrac{P_\star}{1-P_\star}$. Thereby $P(g(\hat{P}^{(T)}_\star)=0)=1$ and $\hat{P}^{(T)}_\star$ is a complete sufficient statistic by definition. By Lehmann-Scheffe Theorem, $\hat{P}^{(T)}_\star$ is the unique UMVUE.
\end{proof}

The complexity of Linear Four Rates (LFR) detection algorithm is $O(1)$ at each time step. The LFR algorithm can be optimized by using a $BoundTable$ precomputed by Algorithm \ref{BoundTable}. 
The 4 dimensional $BoundTable$ with varying input $(\hat{P}, \eta, \delta, N_\star)$ can itself be precomputed and stored 
before running Algorithm \ref{FourRatesAlg}. It is unnecessary to spend any computational resource on quantiles calculation during stream monitoring because observer can find a closest $\hat{P}$ to $\hat{P}^{(t)}_\star$ from $BoundTable$ to look up lower and upper quantiles. Thus, LFR algorithm takes O(1) to test drift occurrence at each time point and suits with streaming environment.

\subsection{Na\"{i}ve Four Rates algorithm (NFR)}
\label{sec:NFRalgorithm}
For the purpose of comparison, this section details the characteristics of a na\"{\i}ve implementation of the proposed framework that uses $\hat{P}^{(t)}_\star$ as the test statistic. A benefit of choosing this test statistic, is that there exists a closed-form distribution as shown in Lemma \ref{UMVUElemma}. Using the same strategy of LFR algorithm, NFR algorithm monitors the four rates $\hat{P}^{(t)}_\star$ sequentially. At each time stamp, for each rate, hypothesis testing is done with null distribution $N(P_\star, \dfrac{P_\star(1-P_\star)}{N_\star})$ and the warning / detection alarms set when $\hat{P}^{(t)}_\star$ exceeds the expected bounds. 

The main difference with respect to LFR is the estimation of $P_\star$ used to find null distribution. LFR algorithm uses $\hat{P}^{(t)}_\star$ as a surrogate of unknown $P_\star$ while NFR algorithm uses $\bar{P}^{(t)}_\star$, where $\bar{P}^{(t)}_\star$ is a running average of all previous $\hat{P}^{(t)}_\star$. This update rule allows old prediction performance contributes more to the estimate of $P_\star$ and recent predictions contributes less. Thus, $\bar{P}^{(t)}_\star$ is more robust in terms of estimating the underlying $P_\star$ when concept drift occurs. Additionally, $\bar{P}^{(t)}_\star$ is still a MSE-consistent estimator  under the stable concept presented in Lemma \ref{NFRconsistency}. 

\begin{lemma}
\label{NFRconsistency}
Assume the setting in Theorem \ref{LFRnulldist}. Under the stable concept up to $T$, for any $\star\in\{tpr, tnr, ppv, npv\}$, $\bar{P}_\star^{(T)}$ in NFR algorithm is a MSE-consistent estimator of $P_\star$. 
 \end{lemma} 
\begin{proof}
Among total $T$ time steps, suppose $\hat{P}^{(t)}_\star$ is changed at time step $T_1,\dots, T_{N_\star}$ where $T_1<T_2<\dots<T_{N_\star}\leq T$. Hence, 
\begin{equation*} \label{eq1}
\begin{split}
\bar{P}^{(T)}_{\star} = &\bar{P}^{(T_{N_\star})}_{\star}  
 = \dfrac{1}{N_\star}\sum_{i=1}^{N_\star}\hat{P}^{(T_i)}_\star\\
  = &\dfrac{1}{N_\star}[\sum_{n=1}^{N_\star}\dfrac{1}{n}\sum_{i=1}^{n}\mathbf{1}{\{y_{T_i}=\hat{y}_{T_i}\}}]\\
= & \dfrac{1}{N_\star}[\sum_{i=1}^{N_\star}\sum_{j=i}^{N_\star}\dfrac{1}{j}\mathbf{1}{\{y_{T_i}=\hat{y}_{T_i}\}}].
\end{split}
\end{equation*}
By IID assumption of indicators, we obtain
\begin{equation*}
\mathbb{E}(\bar{P}^{(T)})=\dfrac{1}{N_\star}[\sum_{i=1}^{N_\star}\sum_{j=i}^{N_\star}\dfrac{1}{j}P_\star]=P_\star
\end{equation*}
and 
\begin{equation*}
\begin{split}
\mathbb{VAR}(\bar{P}^{(T)})= &\dfrac{1}{N^2_\star}[\sum_{i=1}^{N_\star}\Big(\sum_{j=i}^{N_\star}\dfrac{1}{j}\Big)^2P_\star(1-P_\star)] \\
	\leq &\dfrac{(\sum_{j=1}^{N_\star}\dfrac{1}{j})^2}{N_\star}P_\star(1-P_\star) \\
	\leq &\dfrac{(log N_\star)^2}{N_\star} P_\star(1-P_\star)  \xrightarrow{T\to \infty} 0
\end{split}
\end{equation*}
where the last limit hold by the fact that $N_\star\to\infty$ as $T\to\infty$. Thus, as $\mathbb{E}(\bar{P}^{(T)}_{\star})-P_\star=0$ and $\mathbb{VAR}(\bar{P}^{(T)})\to0$, $\bar{P}^{(T)}$ is a MSE-consistent estimator.
\end{proof}

\subsection{Comparison between NFR and LFR}
\label{sec:NFRvsLFR}
To empirically compare the test statistics of NFR and LFR, we use Figure \ref{fig:LFRvsNFRtoy} to illustrate a single run of both LFR and NFR algorithm on the same synthetic streaming data $\{(y_t,\hat{y}_t)\}_{t=1}^{T}$. The data stream of pairs $\{(y_t,\hat{y}_t)\}_{t=1}^{T}$ with one change-point at $T/2$ is generated by sampling from two confusion probability matrices $CP^{(1)}$ and $CP^{(2)}$. The two concepts are characterized by $CP^{(1)}$ and $CP^{(2)}$ respectively. The type of drift is determined by particular settings of $(CP^{(1)}, CP^{(2)})$. In this example, to generate a balanced stream of pairs $\{(y_t,\hat{y}_t)\}_{t=1}^{T}$ representing the scenario that overall accuracy of classifier drops but $P_{tpr}$ remains constant, we chose 
\begin{small}
$CP^{(1)}= \left( \begin{array}{cc}
			0.4 & 0.1 \\
			0.1 & 0.4 \end{array} \right)$ 
and 
$CP^{(2)}= \left( \begin{array}{cc}
			0.3 & 0.1 \\
			0.2 & 0.4 \end{array} \right)$
\end{small}.
The objective of detection algorithms is to identify the change-point $T/2$.

It is clear that the test statistic $R^{(t)}_\star$ in LFR algorithm has a larger variance than $\hat{P}^{(t)}_\star$ for each rate. LFR algorithm reports an earlier detection at $t=5167$ (true detection point t=5000) when compared to NFR in this run, even though $\epsilon_\star^{(LFR)}<\epsilon_\star^{(NFR)}$. This observation matches well with the rationale of constructing $R^{(t)}_\star$, described in \S \ref{LFRoutline}, to gain detection sensitivity through introducing large variances. To rigorously compare detection performance of $R^{(t)}_\star$ and $\hat{P}^{(t)}_\star$, more investigations are provided below.

\begin{figure}[htbp]
	  	\centering
		\hbox{\hspace{0em}\vspace{-1em}
	 	 \includegraphics[scale=0.21]{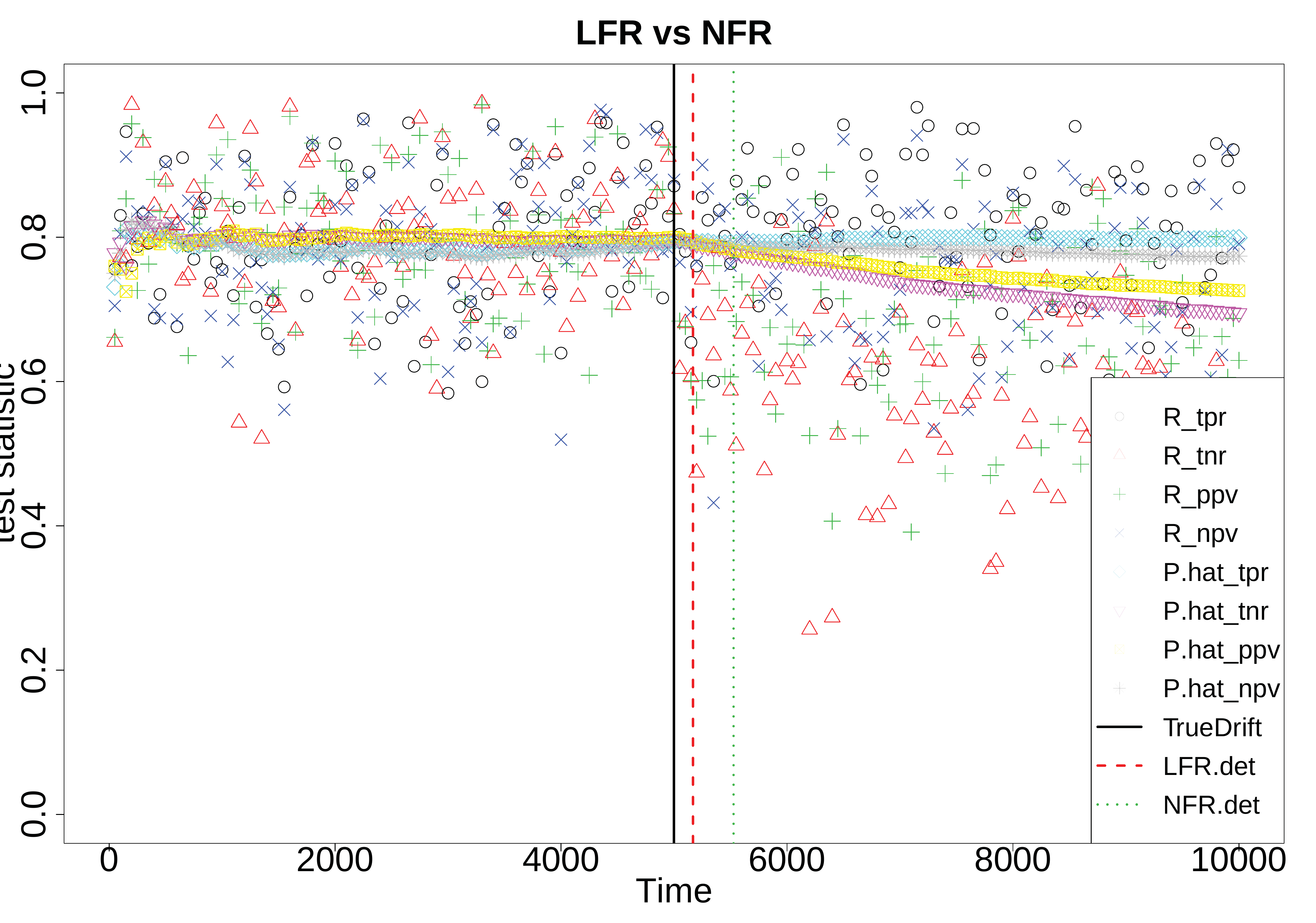}
	 }
	 	 \caption{A single run of LFR and NFR on the same synthetic streaming data. Black, red and green vertical lines are the 'true drift time', LFR detection and NFR detection time respectively. Four colored dots (black, red, green, blue) are running $R^{(t)_\star}$ and four colored horizontal 
lines (indigo, pink, yellow, grey) are running $\hat{P}^{(t)}_\star$ where $\star\in \{tpr,tnr,ppv,npv\}$.}
	 	 \label{fig:LFRvsNFRtoy}
\end{figure}

\label{LFRvsNFRPowerAnalysis}
Power characteristics of two competing test statistics $R^{(t)}_\star$ (LFR) and $\hat{P}^{(t)}_\star$ (NFR), are compared empirically on synthetic data. We denote by $\hat{\beta}_{R^{(t)}_\star}$ and  $\hat{\beta}_{\hat{P}^{(t)}_\star}$ the power estimates of $R^{(t)}_\star$ and $\hat{P}^{(t)}_\star$ respectively. The $\hat{\beta}_{R^{(t)}_\star}$ and  $\hat{\beta}_{\hat{P}^{(t)}_\star}$ against varying time lag $k$ and $q_\star$ are presented in Figure \ref{fig:PowerCompare}. Figure \ref{fig:PowerCompare} indicates that neither $R^{(t)}_\star$ nor $\hat{P}^{(t)}_\star$ dominates all the time because $R^{(t)}_\star$(red surface) achieves a larger statistical power when the time lag $K$ is small but a smaller power when $K$ is large. This is because the update rule line 7 enables the estimator $R^{(t)}_\star$ to shift from $p_\star$ to $q_\star$ at an exponential rate which leads the power dominance in a short lag. The price is that limiting distributions of $R^{(t)}_\star$ under both null and alternative have larger variances than $\hat{P}^{(t)}_\star$ and thus limiting power, when $K$ is large  and $|p_\star-q_\star|$ is small, is degraded.

\begin{figure}[htbp]
	  	\centering
		\hbox{\hspace{1em}\vspace{-1em}  
	 	 \includegraphics[scale=0.34]{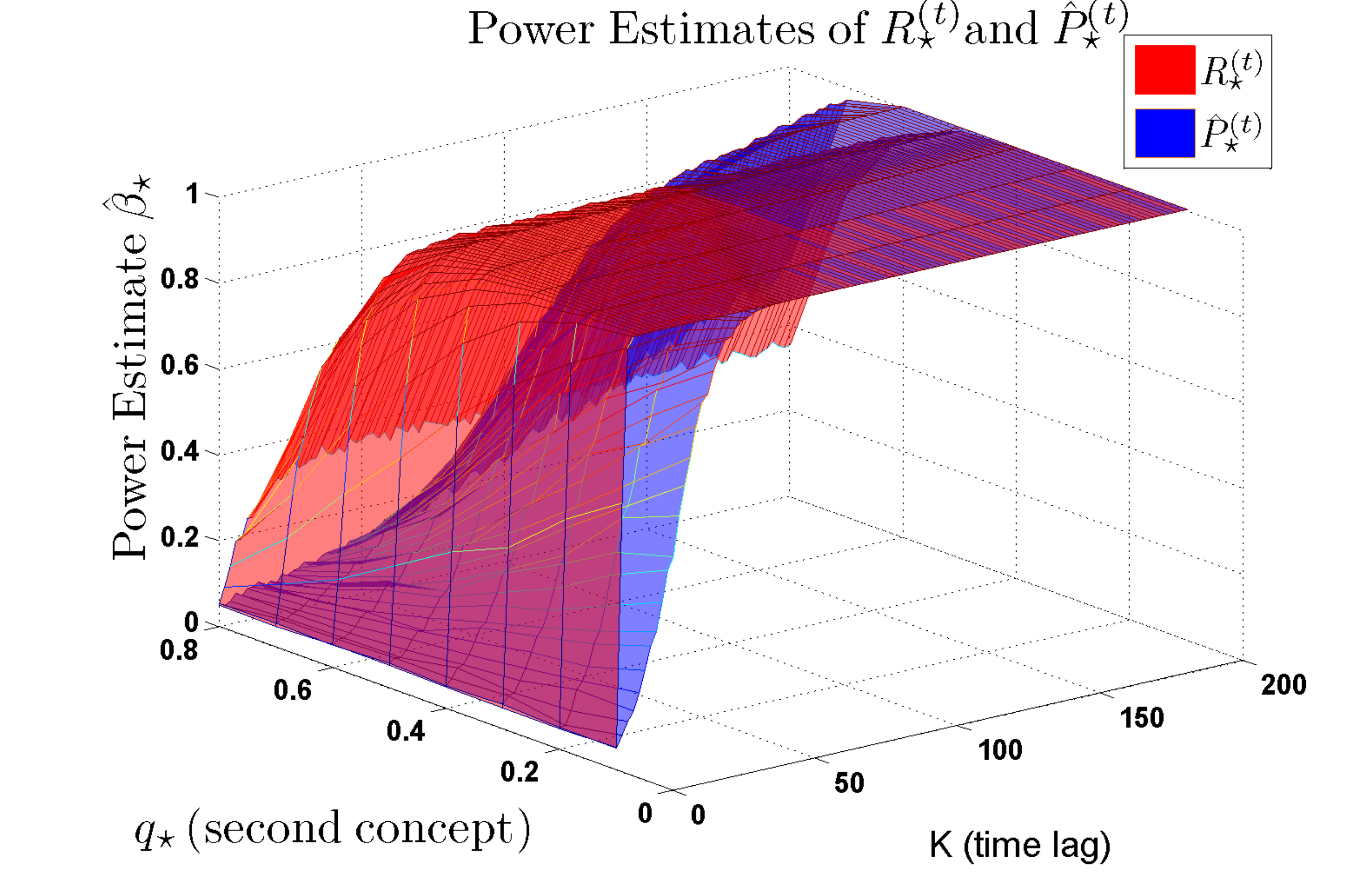}
	 }
	 	 \caption{Power comparison between $R^{(t)}_\star$ and $\hat{P}^{(t)}_\star$ where null distribution is at $t=M$ and alternative distribution is at $t=M+k$. $M=1000$, $1\leq k \leq K$ where $K=200$ is the maximal time lag. The underlying rate is drifted from $p_\star=0.9$ to $q_\star$ where $0.1\leq q_\star\leq 0.8$.}
	 	 \label{fig:PowerCompare}
\end{figure}

In order to compare sensitivities of $R^{(t)}_\star$ and $\hat{P}^{(t)}_\star$  with regard to detecting concept drift in more general settings, we used $\Delta\beta = \hat{\beta}_{R^{(t)}_\star}-\hat{\beta}_{\hat{P}^{(t)}_\star}$. The result is illustrated in Figure \ref{fig:PowerComparisonAll}. 
\begin{figure}[htbp]
	  	\centering
		\hbox{\hspace{1em} \vspace{-1em}
 		\includegraphics[scale=0.29]{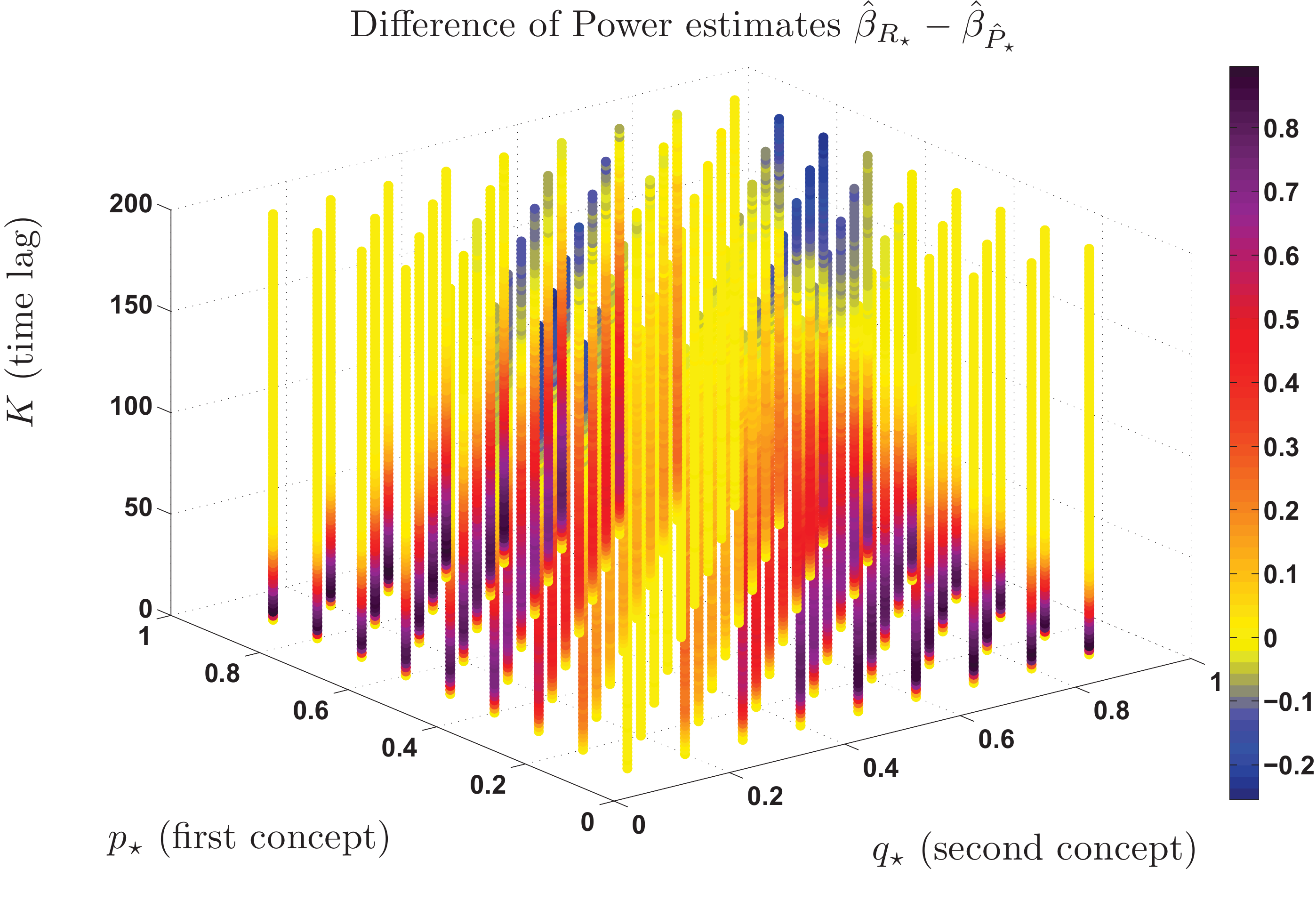}	
	 }
	 	 \caption{Power difference $\Delta\beta=\hat{\beta}_{R^{(t)}_\star}-\hat{\beta}_{\hat{P}^{(t)}_\star}$ along the time lag $K$ in different combinations of concept change from $p_\star$ to $q_\star$.}
	 	 \label{fig:PowerComparisonAll}
\end{figure}
Except when, $p_\star=q_\star$, we see that for any fixed pair of $(p_\star, q_\star)$, $\Delta \beta >0 $ when $K$ is small and $\Delta\beta\leq 0$ when $K$ is large. This is because $\Delta\beta$ decreases, as time lag $K$ increases. This suggests that LFR is preferable if earlier detection is highly desired. The alarms are more likely to be triggered in the earliest time after the occurrence of concept drift. Earlier detection allows observer to adjust the model and avoid costs of incorrect predictions immediately. On the other hand, if observers are only concerned with detecting the occurrence of drift in the data stream but unconcerned with its detection promptness, then NFR algorithm provides a higher power test statistic to detect the drift. This is because $\hat{P}^{(t)}_\star \to q_\star$ with convergence rate $O(\dfrac{1}{K})$. In the long run, as $K\to \infty$, $\hat{P}^{(t)}_\star \to q_\star$ implies that $\beta_{\hat{P}_\star^{(t)}} \to 1$. 

To guide the selection between LFR and NFR, Figure \ref{fig:Heatmap_LFR_K200} is a heatmap of limiting power estimates on all $(p_\star,q_\star)$ pairs using $K=200$. We can see that $\hat{\beta}_{R_\star}$ is already close to 1 for $K=200$, when $p_\star$ and $q_\star$ are significantly different.

\begin{figure}[htbp]
	  	\centering
		\hbox{\hspace{-2.5em} \vspace{-1em}
	 	 \includegraphics[scale=0.43]{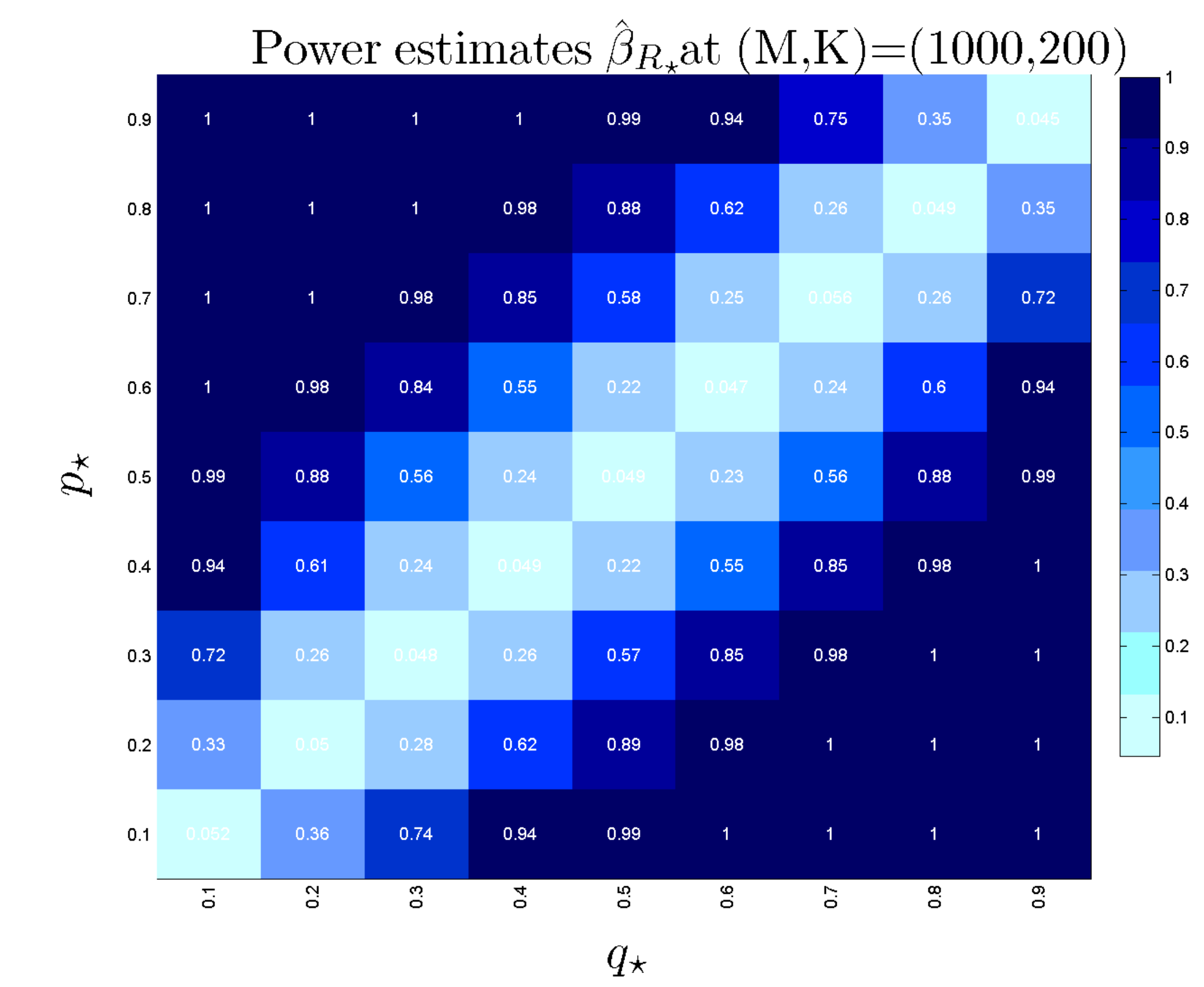}
	 }
	 	 \caption{Heatmap of Power estimates $\hat{\beta}_{R_\star}$} 
	 	 \label{fig:Heatmap_LFR_K200}
\end{figure}
\section{Experiments}
\label{sec:Experiments}
In this section, we compared the detection performance of LFR to NFR, DDM and DDM-OCI approaches using both synthetic data and public datasets. We considered $3$ simulated class-balance datasets, $3$ simulated class-imbalance datasets and $4$ public datasets to demonstrate LFR algorithm performs well across various types of concept drifts, including those where the baseline performs poorly. 

To generalize the performance and evaluate confidences of algorithms, we utilize the bootstrapping technique. For each synthetic dataset, we generate $100$ data streams of $\{(y_t, \hat{y}_t)\}_{t=1}^T$ rather than $\{(\mathbf{X}_t,y_t)\}_{t=1}^T$ so that comparison of detection algorithms is independent of classifiers employed; For each public dataset, the order of $(\mathbf{X}_t,y_t)$ pairs within each concept are permutated to create 100 bootstrapped dataset streams. Each stream is fed to all detection algorithms to obtain single-run detections for each method. To illustrate the accuracy of the prediction, we use overlapped histograms to visualize the distribution of detection points obtained from the concept drift detection models across the 100 runs. To avoid redundancy, we present $6$ histograms out of $10$ experiments and remaining ones are similar. As shown below, LFR consistently outperformed the baseline approaches. When compared to NFR, LFR correctly identifies more true drift points with higher probability and smaller number of false alarms even with a smaller $\epsilon_\star$.

\subsection{Synthetic Data}
\label{sec:exp-synthetic}
Numerous experiments were run on synthetic data, covering various types of concept drift. 
In each bootstrap, a data stream of pairs $\{(y_t,\hat{y}_t)\}_{t=1}^{T}$ with one change-point at $T/2$ is generated by using the same mechanism introduced in \S\ref{sec:NFRvsLFR}. The objective of detection algorithms is to identify the change-point $T/2$.
Six challenging and interesting scenarios are discussed below.

\subsubsection{Balanced Dataset}
\label{sec:exp-balance}
In balanced datasets, $P(y_t=0)=P(y_t=1)$ is required in underlying data generation. Class-balance data are the most typical scenario in classification task and hence investigated with following three representative experiments.
\begin{enumerate} [(i)]
\item Balance1: Overall accuracy of classifier drops but $P_{tpr}$ remains constant with
\begin{small}
$CP^{(1)}= \left( \begin{array}{cc}
			0.4 & 0.1 \\
			0.1 & 0.4 \end{array} \right)$ 
and 
$CP^{(2)}= \left( \begin{array}{cc}
			0.3 & 0.1 \\
			0.2 & 0.4 \end{array} \right)$
\end{small}.
\item Balance2: Gradual drift in which overall accuracy ($1-P_{error}$) remains the same with 
\begin{small}
$CP^{(1)}= \left( \begin{array}{cc}
				0.35 & 0.05 \\
				0.15 & 0.45 \end{array} \right),$ 
	$CP^{(2)}= \left( \begin{array}{cc}
				0.4 & 0.1 \\
				0.1 & 0.4 \end{array} \right)$ 
\end{small}. 
\item Balance3: Overall accuracy ($1-P_{error}$) increases and $P_{tpr}$ remains unchanged with 
\begin{small}
$CP^{(1)}= \left( \begin{array}{cc}
			0.3 & 0.2 \\
			0.2 & 0.3 \end{array} \right)$ 
,
$CP^{(2)}= \left( \begin{array}{cc}
			0.4 & 0.2 \\
			0.1 & 0.3 \end{array} \right)$
\end{small}.
\end{enumerate}
\begin{figure}
	  	\centering
		\hbox{\hspace{-0.5em}\vspace{-1em}
	 	 \includegraphics[scale=0.31]{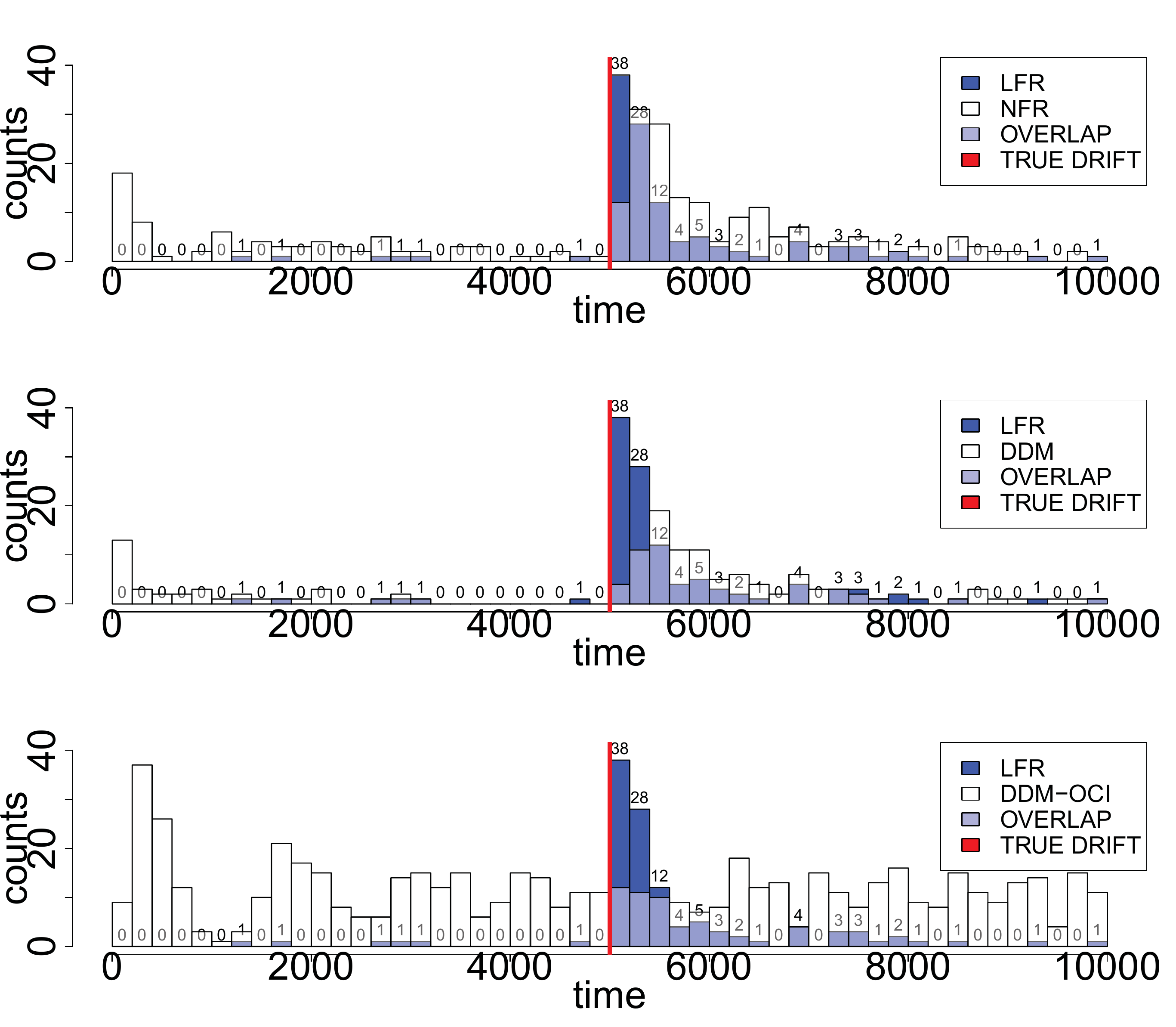}
	 }
	 	 \caption{Overlapping histograms comparing detection timestamps on Balance1 dataset in which overall accuracy of classifier drops but $P_{tpr}$ remains unchanged. Number of counts of LFR is above the top bar of each bin.}
	 	 \label{fig:balance1}
\end{figure}
\begin{figure}
	  	\centering
		\hbox{\hspace{1em}\vspace{-1em}
	 	 \includegraphics[scale=0.3]{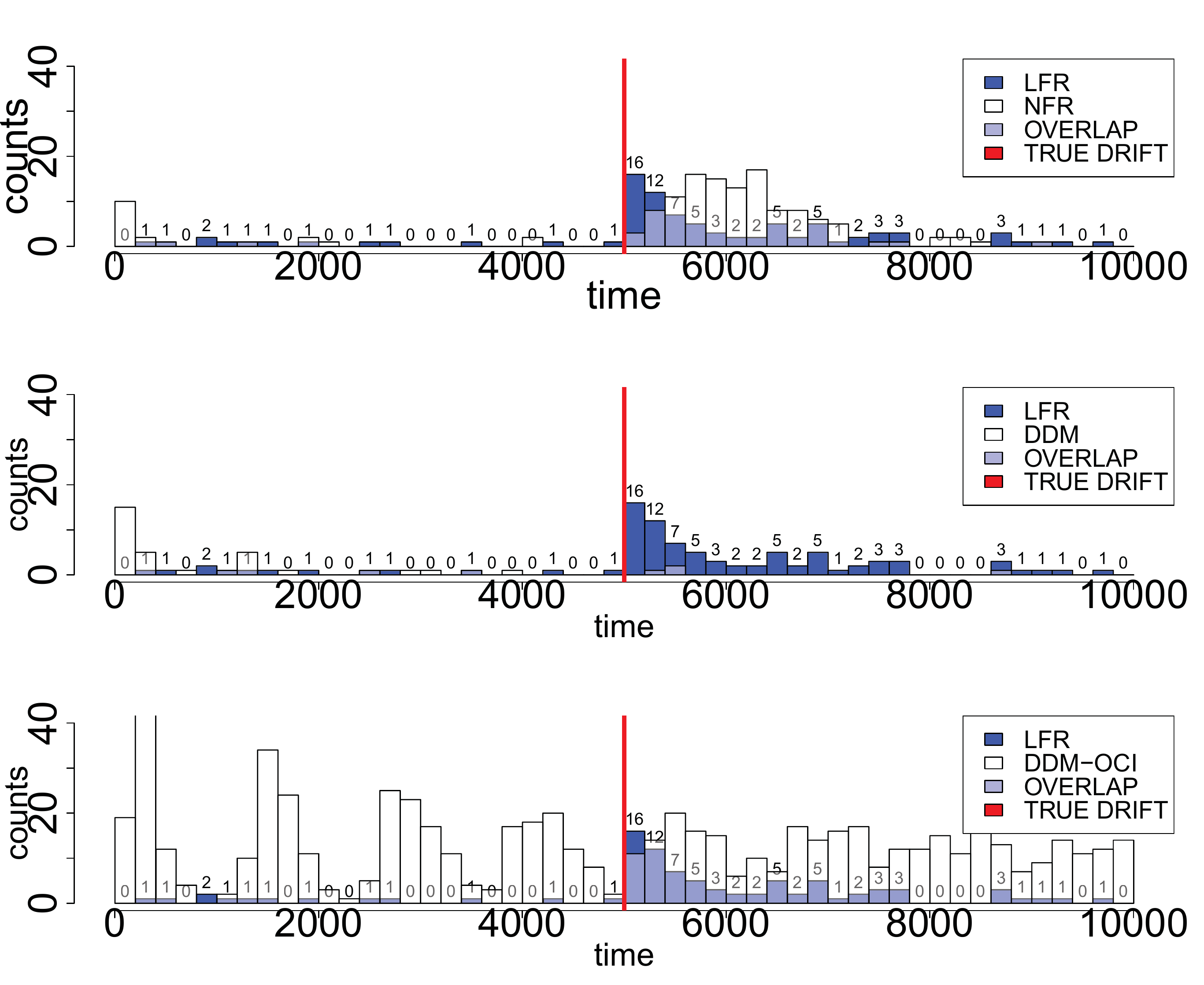}
	 }
	 	 \caption{Overlapping histograms comparing detection timestamps on Balance2 dataset in which gradual drift occurs but overall accuracy remains the same. Number of counts of LFR is above the top bar of each bin.}
	 	 \label{fig:balance2}
\end{figure}



\subsubsection{Imbalanced Dataset}
\label{sec:exp-imbalance}
\begin{figure}
	  	\centering
		\hbox{\hspace{1em}\vspace{-1em}
	 	 \includegraphics[scale=0.3]{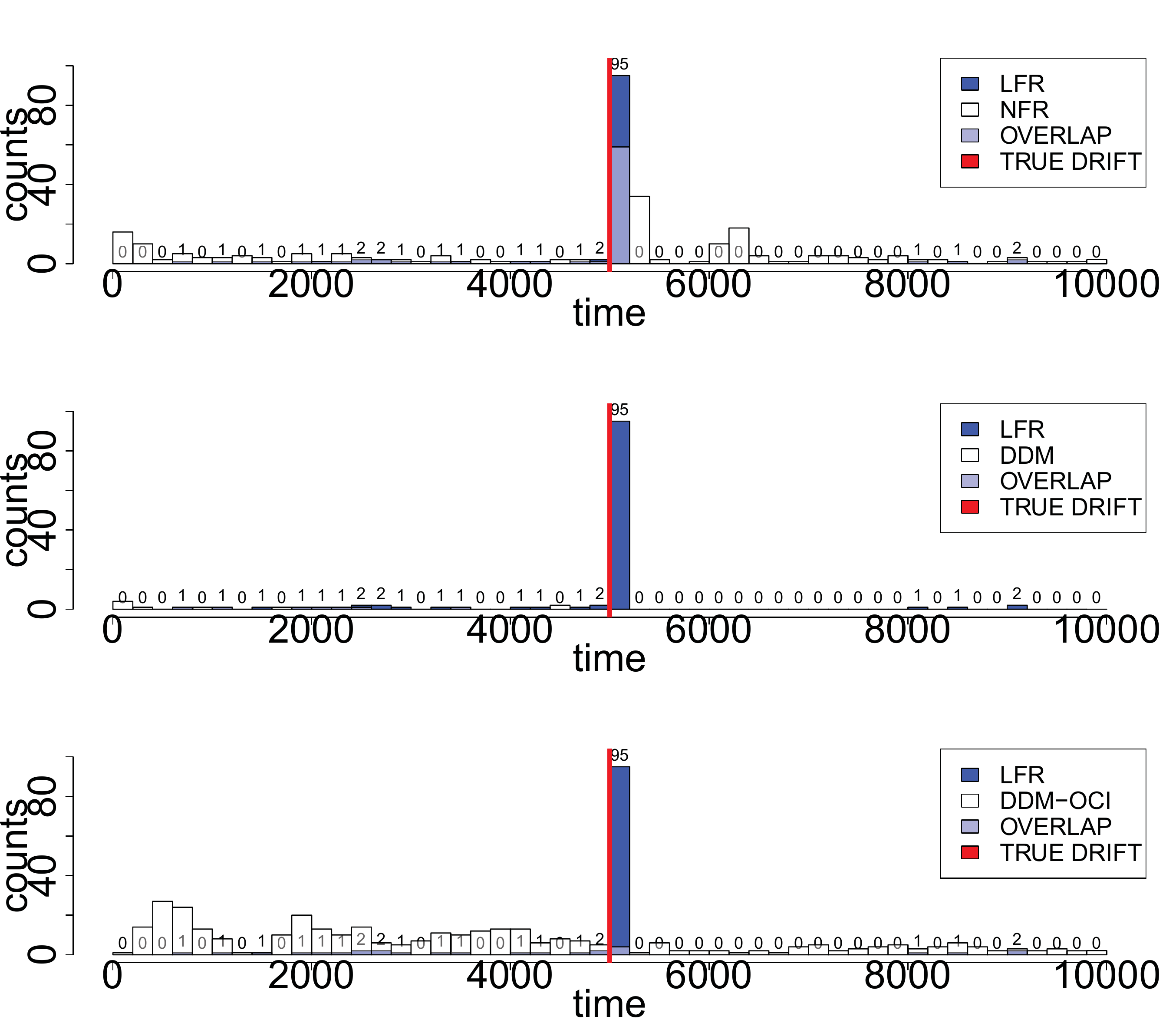}
	 }
	 	 \caption{Overlapping histograms comparing detection timestamps on Imbalance1 dataset in which class ratio transits from 1:1 to 9:1 but $F1-$score remains unchanged. Number of counts of LFR is above the top bar of each bin.}
	 	 \label{fig:imbalance1}
\end{figure}
For imbalanced datasets, we used the same data generation mechanism as balanced case but make $P(y_t=0)$ and $P(y_t=1)$ imbalanced. We considered the following three interesting types of concept drifts given many attentions to in real applications.
\begin{enumerate} [(i)]
\item Imbalance1: From class balance dataset to class imbalance dataset with
\begin{small}
$CP^{(1)}= \left( \begin{array}{cc}
			1/3 & 1/6 \\
			1/6 & 1/3 \end{array} \right)$ 
and
$CP^{(2)}= \left( \begin{array}{cc}
			13/15 & 1/30 \\
			1/30 & 1/15 \end{array} \right)$
\end{small}. 
Without loss of generality, let $y=1$ be the minority class. It is also noteworthy that $P_{tpr}$ and $P_{ppv}$ are unchanged after drift occurrence. Hence, many detectors in imbalance data learning society, using F1 score as a measure to monitor classifier performance, is unable to alarm this type of drift. However, \figurename~\ref{fig:imbalance1} shows that LFR performs very well by dominating both high early detection rate and trivial false alarms. Besides, DDM and DDM-OCI has no detection after change-point due to the increment of ($1-P_{error}$) and $P_{tpr}$, respectively.

\item Imbalance2: The class ratio and $P_{error}$ remain unchanged but $P_{tpr}$ decreases with
\begin{small}
$CP^{(1)}= \left( \begin{array}{cc}
			0.65 & 0.05 \\
			0.15 & 0.15 \end{array} \right)$ 
and
$CP^{(2)}= \left( \begin{array}{cc}
			0.75 & 0.15 \\
			0.05 & 0.05 \end{array} \right)$
\end{small}. 


\item Imbalance3: All $P_{tpr}$,$P_{ppv}$ and $1-P_{error}$ decreases. Though class ratio remains the same, both F1-score and overall accuracy decreases.
 Two conditional probability matrices are selected as 
\begin{small}
$CP^{(1)}= \left( \begin{array}{cc}
			0.6 & 0.15 \\
			0.15 & 0.1 \end{array} \right)$ and 
$CP^{(2)}= \left( \begin{array}{cc}
			0.6 & 0.15 \\
			0.15 & 0.1 \end{array} \right)$
\end{small}. 			
\end{enumerate}
\subsection{Public Datasets}
All detection algorithms are evaluated on four public datasets used in literature. Without loss of generality, we chose the Support Vector Machine (SVM)\cite{meyer2014support} with an RBF Kernel as the classifier $\hat{f}$, because all detection algorithms are independent of type of classifiers. Misclassification of the minority class is penalized 100 times more than the majority class. If a potential concept drift is reported by the algorithm, examples from the new concept will be stored to retrain a new SVM classifier $\hat{f}_{new}$, adapted with new concept. Specifically, $1000$ examples are used for retraining on SEA and Rotating Hyperplane datasets; $100$ examples are used for retraining on USENET1 and USENET2 datasets.

\label{sec:exp-public}
\subsubsection{Datasets}
\label{sec:exp-public-dataset}
\begin{small}
\begin{table}
\centering
    \begin{tabular}{| c | c | c | c |}
    \hline
    Dataset & $T$ & True Drift Time & dimensions ($d$)  \\ \hline
    SEA & $60000$ & $\{15000\times i\}_{i=1}^{3}$ & $3$ \\ \hline
    HYPER. & $ 90000$ & $\{10000\times i\}_{i=1}^{8}$ & $10$\\ \hline
    USENET1 & $1500$ & $\{300\times i\}_{i=1}^{5}$ &100 \\ \hline
    USENET2 & $1500$ & $\{300\times i\}_{i=1}^{5}$ &100 \\ \hline
    \end{tabular}
\caption{Key features of datasets.}
\label{tab:public-dataset}
\end{table}
\end{small}
SEA Concepts dataset is used in \cite{street2001streaming}. The dataset is available at \url{http://www.liaad.up.pt/kdus/products/datasets-for-concept-drift}, and is widely used as a testbed by concept drift detection algorithms. Rotating Hyperplane dataset is created by \cite{fan2004systematic}.The dataset and specific $(k,t)$ pairs of each concept are available at \url{http://www.win.tue.nl/~mpechen/data/DriftSets/}. USENET1 and USENET2 datasets, used in \cite{katakis2008ensemble}, are available at \url{http://mlkd.csd.auth.gr/concept_drift.html}. They are stream collections of messages from different newsgroups (e.g. medicine, space, baseball) to a user. The difference between USENET1 and USENET2 is the magnitude of drift. The user in USENET1 has a sharper topic shift than the one in USENET2. 
\begin{figure}
		\hbox{\hspace{0em}\vspace{-1em}
	 	 \includegraphics[scale=0.33]{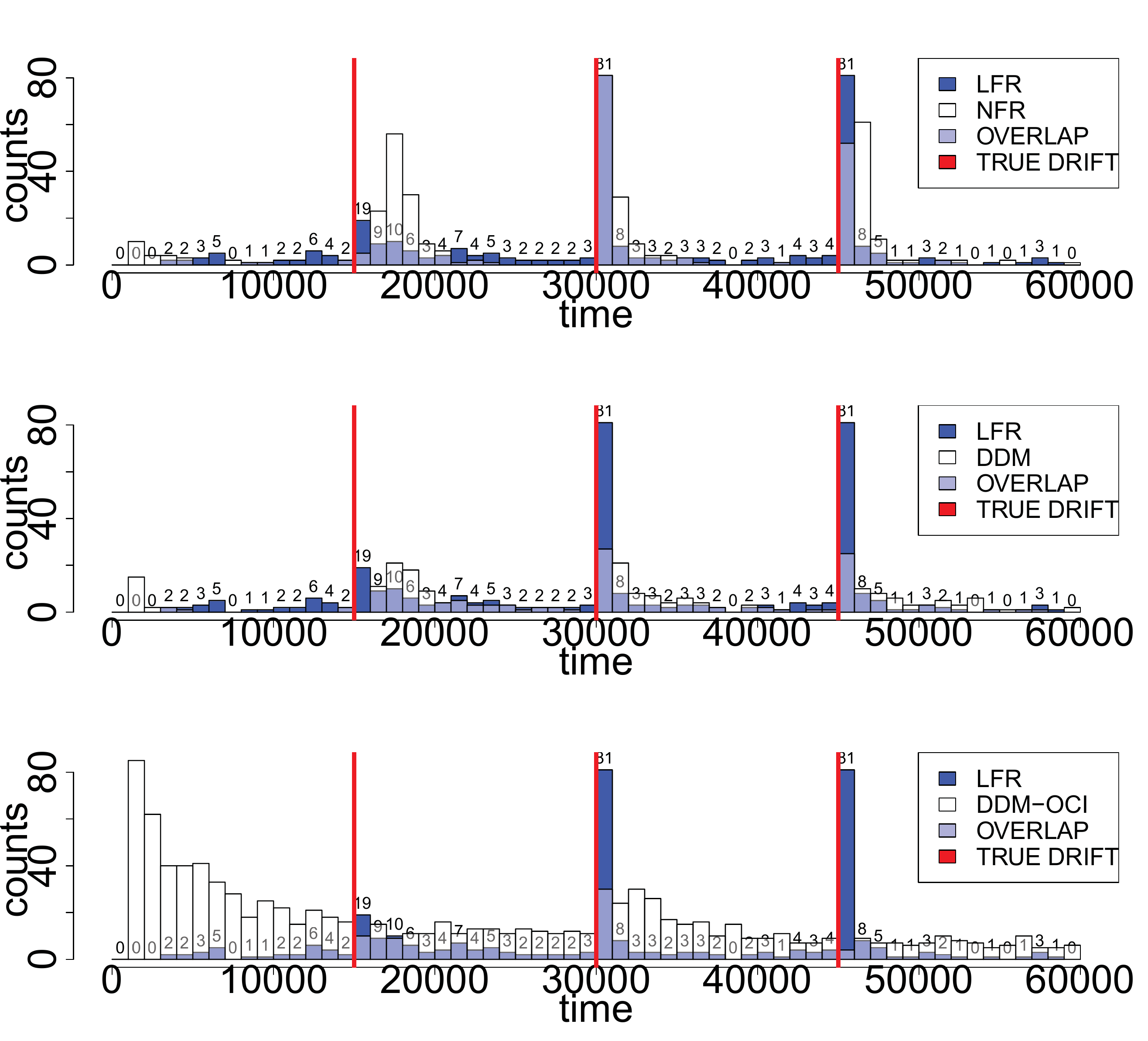}
	 }
	 	 \caption{Overlapping histograms comparing detection timestamps on SEA.}
	 	 \label{fig:SEA}
\end{figure}

\begin{figure}
		\hbox{\hspace{-1em}\vspace{-1em}
	 	 \includegraphics[scale=0.25]{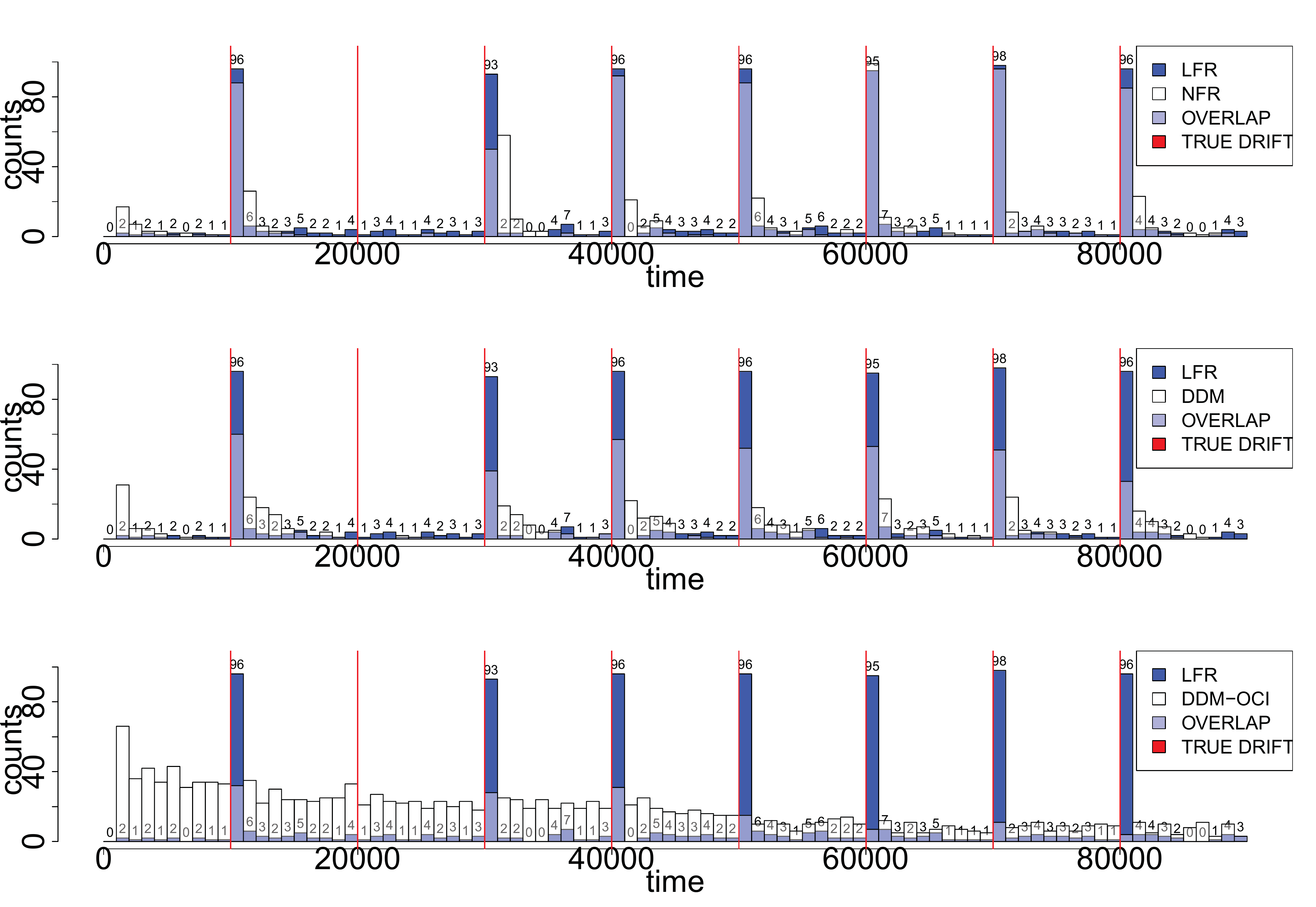}
	 }
	 	 \caption{Overlapping histograms comparing detection timestamps on HYPERPLANE.}
	 	 \label{fig:HYPERPLANE}
\end{figure}

\begin{figure}
		\hbox{\hspace{0em}\vspace{-1em}
	 	 \includegraphics[scale=0.33]{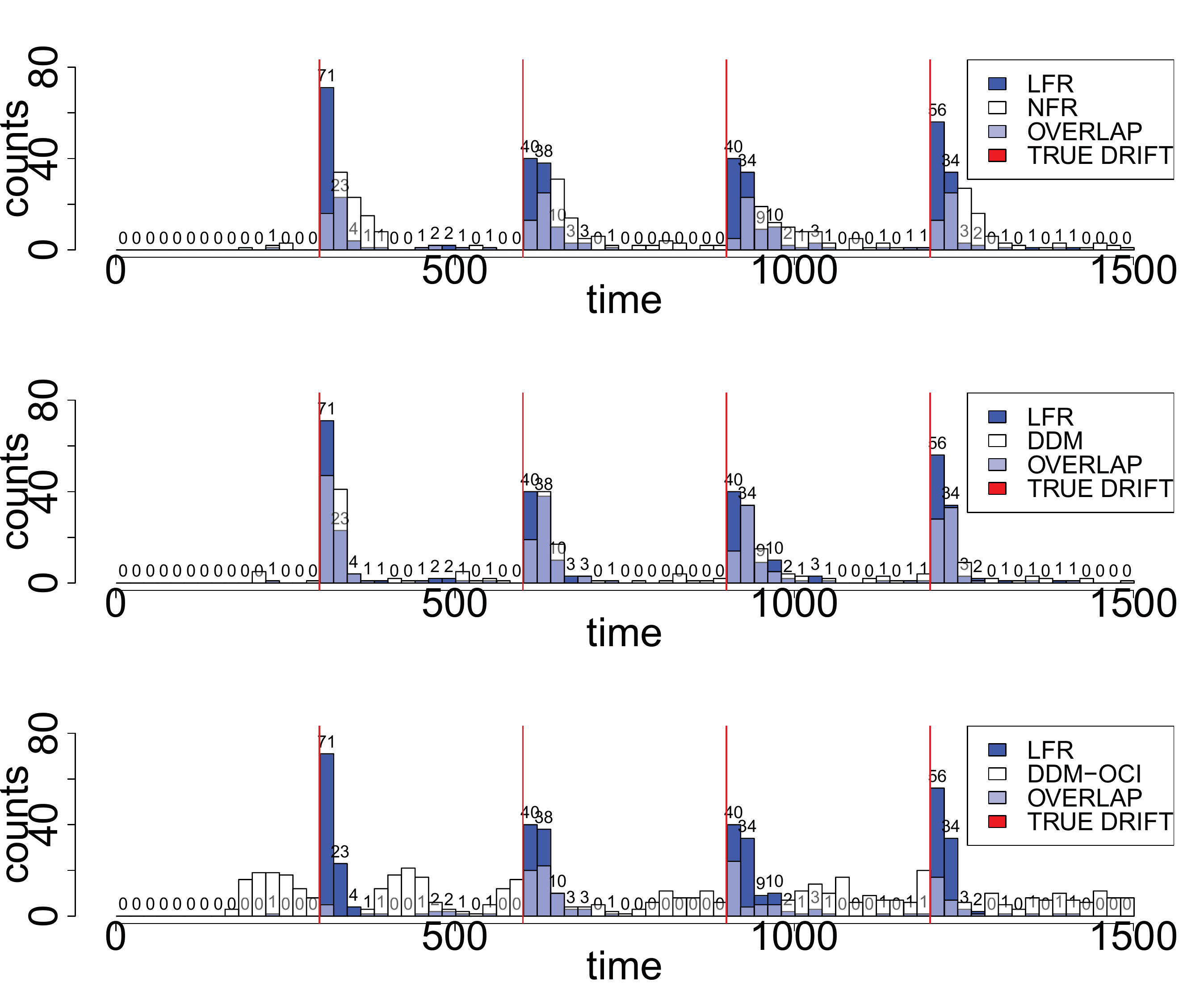}
	 }
	 	 \caption{Overlapping histograms comparing detection timestamps on USENET1.}
	 	 \label{fig:USENET1}
\end{figure}
All above datasets are in form of $\{\mathbf{X}_t,y_t\}_{t=1}^{T}$ and their key features are summarized in Table \ref{tab:public-dataset}. Other details such as the imbalance status and type of drift of each dataset are available through above links.
\begin{small}
\begin{table}
\centering
    \begin{tabular}{| c | c | c | c |c |}
    \hline
    Metric & LFR & NFR & DDM & DDM-OCI \\ \hline
    Balance1 & 38 &  12 & 4 & 12 \\ \hline
    Balance2 & 16 & 3 & 0 & 11 \\ \hline
    Balance3 & 25 & 4 & 0 & 3  \\ \hline
    Imbalance1 & 95 & 59 & 0 & 4 \\\hline
    Imbalance2 & 91 & 21 & 0 & 43 \\\hline
    Imbalance3 & 95 & 38 & 36 & 39 \\\hline
    SEA  & 142 & 29 & 17 & 26 \\\hline 
    HYPRPLN & 671 & 598 & 345 & 149 \\\hline
    USENET1 & 207 & 47 & 108 & 66 \\\hline
    USENET2 & 3 & 17 & 3 & 21 \\\hline
    \end{tabular}
\caption{The count (sum) at (multiple) true drift point correctly detected for simulated (public) datasets.}
\label{tab:counts of true drift point}
\end{table}
\end{small}
\begin{small}
\begin{table}
\centering
    \begin{tabular}{|l | l |l | l | l |}
    \hline
    Metric & LFR & NFR & DDM & DDM-OCI\\ \hline
    Balance1 & 6 & 77 & 36 & 304\\ \hline
    Balance2 & 13 & 19 & 33 & 339\\ \hline
    Balance3 & 18 &  54 & 11 & 219\\ \hline
    Imbalance1 & 18 & 81& 16 & 259\\\hline
    Imbalance2 & 10 & 91 & 23 & 165\\\hline
    Imbalance3 & 9 & 86 & 55 & 204\\\hline
    SEA  & 72 & 32 & 54 & 658\\\hline 
    HYPRPLN & 84 & 56 & 73 & 826\\\hline
    USENET1 & 12 & 50 & 43 & 322\\\hline
    USENET2 & 43 & 80 & 65 & 272\\\hline
    \end{tabular}
\caption{The count (sum) at (multiple) false detection for the simulated (public) datasets}
\label{tab:counts of false detection}
\end{table}
\end{small}
\subsubsection{Evaluation}
\label{sec:exp-public-evaluation}
In SEA Concepts Dataset experiment, \figurename~\ref{fig:SEA} shows that LFR dominates other three approaches in terms of early detections and fewer false or delayed detections. 

\figurename~\ref{fig:HYPERPLANE} shows that LFR has a dominant performance on the Rotation Hyperplane Dataset experiment. At the second true drift time point, the underlying concept change is very minor. Hence the drift is neglected by all detection algorithms.

In USENET1 dataset experiment, \figurename~\ref{fig:USENET1} indicates LFR dominates other approaches and all drift points are alarmed.  Similarly, in USENET2 dataset experiment, LFR also outperforms other approaches 
but detections are delayed with longer time lag. The decrement of superiority of LFR, from USENET1 to USENET2 is due to decrements of magnitude of concept drifts.

\subsection{summary statistics}
\begin{table}
\centering
    \begin{tabular}{| l | l | l | l | }
    \hline
    Parameters & Detect Sig. & Warn Sig. &  Decay\\ \hline
    LFR & $\epsilon_\star=1/100K$ & $\delta_\star=1/100$ & $\eta_\star=0.9$\\ \hline
    NFR & $\epsilon_\star=1/1K$ & $\delta_\star=0.025$ & $\eta_\star=0.9$\\ \hline
    DDM & $\alpha_{detect}=3$ & $\alpha_{warn}=2$ & $\eta_\star=0.9$ \\ \hline
    DDM-OCI & $\alpha_{detect}=20$ & $\alpha_{warn}=10$ & $\eta_\star=0.9$ \\\hline
    \end{tabular}
\caption{Parameter settings used in \S \ref{sec:exp-synthetic} experiments}
\label{tab:parameter-set}
\end{table}

\begin{table}
 	\centering
    \begin{tabular}{| p{0.8cm} | p{1.8cm} | p{1.8cm} | p{1.8cm} | }
    \hline
    Para. & SEA & HYPRPLN. &  USENET1\&2 \\ \hline
    	LFR & $\epsilon_\star=1/10K$ $\delta_\star=1/100$ & $\epsilon_\star=1/10K$ $\delta_\star=1/100$ & $\epsilon_\star=1/10K$ $\delta_\star=1/100$\\ \hline
 	NFR & $\epsilon_\star=1/1K$ $\delta_\star=0.025$ &  $\epsilon_\star=1/1K$ $\delta_\star=0.025$  & $\epsilon_\star=1/1K$ $\delta_\star=0.025$\\ \hline
 	DDM & $\alpha_{detect}=3$ $\alpha_{warn}=2$  &  $\alpha_{detect}=3$ $\alpha_{warn}=2$ & $\alpha_{detect}=3$ $\alpha_{warn}=2$\\ \hline
 	DDM-OCI & $\alpha_{detect}=20$ $\alpha_{warn}=10$ & $\alpha_{detect}=30$ $\alpha_{warn}=10$ &$\alpha_{detect}=3$ $\alpha_{warn}=2$\\ \hline
    \end{tabular}
\caption{Parameter settings used in \S \ref{sec:exp-public} experiments}
\label{tab:public-parameter-set}
\end{table}

In general, the best algorithm will have the minimal number of false alarms and maximal number of early detections, whereas poor algorithms give large number of false alarms, missing or severely delayed true detections. A summary of the counts of correct detections at true drift timestamp and counts of false detections during false detection period for the simulated and public datasets are provided in Tables \ref{tab:counts of true drift point} and Table \ref{tab:counts of false detection}. 

False detection period refers to the period preceding the data points that belong to the new concept. For the synthetically generated datasets in \S \ref{sec:exp-synthetic}, there were two concepts spanning the $T$ data points, such that the false detection period is defined as $[0,T/2)$. For the datasets specified in \S \ref{sec:exp-public}, if there were more than two concepts, the false detection period corresponds to the range from the concept midway up to the next true drift point. Each bin in the histograms correspond to $200$ time steps in \S \ref{sec:exp-synthetic} and dataset-dependent in \S \ref{sec:exp-public}. 
Since it has been observed in \cite{alippi2013just, minku2012ddd} that false alarms may have a smaller influence on predictive performance than late drift detections, the true detection period in our experiments refers to the period spanning next $200$ time steps ($1$ bin) after $T/2$ in \S \ref{sec:exp-synthetic} and the period spanning $1$ bin after each true drfit point in \S \ref{sec:exp-public}.
Other parameter settings of detection algorithms are summarized in Table \ref{tab:parameter-set} and \ref{tab:public-parameter-set}. They are particularly selected to show the dominating performance of LFR, i.e. the smallest allowable type-I error but the largest statistical power, over benchmark algorithms.

As sumarized in Table \ref{tab:counts of true drift point}, LFR fared best in terms of recall of true change point detecion across the various datasets. Equally importantly, LFR had the the highest precision with regard to detecting change points by producing the least amount of false detection and delayed detection (Table \ref{tab:counts of false detection}).

\section{Conclusion}
The paper presents a concept drift detection framework (LFR) for detecting the occurance of a concept drift and identifies the data points that belong to the new concept. The versitality of LFR allows it to work with both batch and stream datasets, imbalanced data sets and it uses user-specified parameters that are intuitively comprehensible, unlike other popular concept drift detection approaches. LFR significantly outperforms existing benchmark approaches in terms of early detection of concept drifts, high detection rate and low false alarm rate across the types of concept drifts.
\label{sec:conclusion}	


%
\bibliographystyle{IEEEtran}
\bibliography{WA}

\end{document}